\documentclass{article}

\usepackage[preprint]{neurips_2026}

\usepackage[utf8]{inputenc} 
\usepackage[T1]{fontenc}    
\usepackage{hyperref}       
\usepackage{url}            
\usepackage{booktabs}       
\usepackage{amsfonts}       
\usepackage{nicefrac}       
\usepackage{microtype}      
\usepackage{xcolor}         
\usepackage{comment}
\usepackage{graphicx}
\usepackage{caption}
\usepackage{amsmath}
\usepackage{subcaption}
\usepackage{multirow}

\usepackage{amsthm}

\newtheorem{theorem}{Theorem}

\newtheorem{proposition}[theorem]{Proposition}

\theoremstyle{definition}

\theoremstyle{remark}

\title{StableGrad: Backward Scale Control\\without Batch Normalization}

\author{%
  Jose I. Mestre \\
  Universitat Politècnica de València\\
  \texttt{jimesmir@upv.es} \\
  \And
  Alberto Fernández-Hernández \\
  Universitat Politècnica de València\\
  \texttt{a.fernandez@upv.es} \\
  \And
  Cristian Pérez-Corral \\
  Universitat Politècnica de València\\
  \texttt{cpercor@upv.es} \\
  \And
  Manuel F. Dolz \\
  Universitat Jaume I\\
  \texttt{dolzm@uji.es} \\
  \And
Enrique S. Quintana-Ortí \\
  Universitat Politècnica de València\\
  \texttt{quintana@disca.upv.es} \\
}

\begin{document}

\maketitle

\begin{abstract}
Training very deep neural networks requires controlling the propagation of magnitudes across depth. Without such control, activations and gradients may vanish, explode, or enter unstable regimes that make optimization fail. Modern architectures often mitigate this problem through Batch Normalization, residual connections, or other normalization layers, which repeatedly re-scale or bypass intermediate representations. However, these mechanisms are not always appropriate. In Physics-Informed Neural Networks (PINNs), the network represents a continuous physical field and its input derivatives define the training objective, making batch-dependent normalization problematic because it can introduce non-local dependencies into the predicted field and its derivatives. We propose StableGrad, an optimizer-level scale-control mechanism that corrects layer-wise weight-gradient imbalances without modifying the forward model. Because the normalization is applied only after backpropagation and before the optimizer update, the network output, its derivatives, and the physical residual remain unchanged. We analyze the effective training dynamics induced by this rescaling and evaluate StableGrad on deep PINNs as the target application, with BatchNorm-free convolutional networks serving as a diagnostic stress test. On PINN benchmarks, StableGrad improves matched-depth solution accuracy and makes deeper models more reliable under standard optimization. On ResNet and EfficientNet architectures, where removing Batch Normalization normally leads to training collapse, StableGrad stabilizes optimization without introducing any other architectural change. These results show that optimizer-level control of weight-gradient scale can provide a practical alternative when forward normalization is unavailable or undesirable.
\end{abstract}

\section{Introduction}
Depth is a central source of expressive power in Deep Neural Networks (DNNs), but it also makes training increasingly sensitive to the propagation of magnitudes across layers. Activations must remain well-scaled in the forward pass, while gradients must remain useful in the backward pass. Although these two requirements are related, they are not equivalent: preserving scale in one direction does not necessarily preserve it in the other.

Classical initialization schemes address this problem by controlling weight variance through choices such as \texttt{fan-in}, \texttt{fan-out}, or combinations between them \citep{glorot2010understanding,he2015delving}. 
These choices implicitly define a trade-off between forward and backward stability: \texttt{fan-in} scaling primarily preserves activation magnitudes in the forward pass, whereas \texttt{fan-out} scaling primarily preserves gradient magnitudes in the backward pass. 
Moreover, initialization only sets the scale of the network at the beginning of training; it does not guarantee that activation or gradient magnitudes will remain well controlled as the weights evolve under optimization. 
Normalization layers, most notably Batch Normalization, reduce this burden by repeatedly re-scaling intermediate representations \citep{ioffe2015batch}. 
However, such mechanisms are not always available. 
In Physics-Informed Neural Networks (PINNs), the network represents a continuous physical field whose derivatives are used to define the training objective \citep{raissi2019physics}. 
Batch-dependent normalization can therefore interfere with the local interpretation of the predicted field and its derivatives.

In this work, we propose StableGrad, an optimizer-level mechanism for controlling backward scale without introducing architectural normalization. 
Rather than relying on initialization to simultaneously balance forward and backward propagation, StableGrad acts after the backward pass and before the optimizer step, directly rescaling the weight gradients of each layer. 
Specifically, each layer-wise weight gradient is normalized using its own empirical standard deviation and rescaled using the standard deviation of the prediction gradient as a reference scale. 
This transfers the scale of the output adjoint to the parameter updates, mitigating depth-induced gradient-scale imbalances across layers.

StableGrad is complementary to initialization rather than a replacement for it. 
In our experiments, we use an activation-aware \texttt{fan-in} initialization to obtain a well-scaled forward pass at the start of training, while StableGrad dynamically controls the backward gradient scale throughout optimization. 
Because the normalization is applied only to gradients before the optimizer update, it does not modify the forward computation, introduce batch-dependent predictions, or alter the physical residual.

Our contributions are:
\begin{itemize}
    \item We introduce StableGrad, an optimizer-level gradient rescaling mechanism that stabilizes layer-wise weight-gradient magnitudes using the prediction-gradient scale as a reference.
    \item We analyze the local training dynamics induced by StableGrad through its effective kernel, showing how layer-wise gradient rescaling changes the functional update.
    \item We evaluate StableGrad in two settings where architectural normalization is undesirable or deliberately removed: BatchNorm-free CNNs, where it stabilizes training without BatchNorm, and deep PINNs, where it enables deeper networks and improves solution accuracy.
\end{itemize}


\section{Background and Motivation}

\subsection{Forward and Backward Signal Propagation}
A deep neural network applies a sequence of transformations of the form
\begin{equation}
    h_{\ell+1} = \phi(z_\ell), 
    \qquad 
    z_\ell = W_\ell h_\ell ,
\end{equation}
where $h_\ell$ denotes the representation at layer $\ell$, $W_\ell$ the corresponding weight matrix, and $\phi$ a nonlinear activation. As depth increases, the scale of $h_\ell$ becomes increasingly sensitive to the statistics of the weights and activations. If this scale is not controlled, representations may progressively vanish, explode, or enter saturated regimes where optimization becomes difficult.

The backward pass has an analogous propagation problem. Gradients are recursively transformed as
\begin{equation}
    \frac{\partial \mathcal{L}}{\partial h_\ell}
    =
    W_\ell^\top
    \left(
    \frac{\partial \mathcal{L}}{\partial h_{\ell+1}}
    \odot \phi'(z_\ell)
    \right),
\end{equation}
where \(\odot\) denotes the Hadamard, or element-wise, product. Thus, the same weights and nonlinearities that determine forward activation scales also affect backward gradient scales. However, the conditions for stable forward propagation and stable backward propagation are not identical. A choice of weight scale that preserves activation magnitudes does not necessarily preserve gradient magnitudes, especially in deep networks. Depth therefore turns scale propagation into a bidirectional constraint: 
activations must remain well-scaled in the forward pass, while the backpropagated adjoint signals must retain useful magnitudes before they induce parameter gradients.

This issue is closely related to the classical vanishing and exploding gradient problem \citep{bengio1994learning}, as well as later analyses of signal propagation, critical initialization, and dynamical isometry in deep networks \citep{saxe2013exact,schoenholz2017deep,pennington2017resurrecting}.

\subsection{Initialization as a Trade-off}
Weight initialization is the first mechanism used to control scale propagation. Classical initialization schemes, such as Xavier and Kaiming, set weight variance according to layer dimensions and activation statistics, aiming to preserve signal variance across depth \citep{glorot2010understanding,he2015delving}. This principle also persists in more recent initializers: even when they replace random sampling with structured constructions, or introduce refined activation-dependent scaling, the final magnitude is still governed by fan-in, fan-out, or combinations of both \citep{Chang2020Principled,fernandezhernandez2025sinusoidal}. Thus, initialization remains tied to a fan-mode choice. Using fan-in primarily favors forward activation stability, whereas fan-out favors backward gradient stability. Averages or other interpolations between fan-in and fan-out reduce the asymmetry, but they do not remove the trade-off; they merely choose a different compromise between the two propagation directions.

Activation-dependent gains further adjust the weight scale to account for the expected effect of nonlinearities. However, these gains do not remove the underlying tension: initialization must still decide how much scale preservation to allocate to the forward pass and how much to the backward pass. This trade-off becomes more restrictive as depth increases, because small deviations from the desired scale can compound across many layers.

This motivates a different view of initialization. Rather than treating it as a single mechanism that must balance forward and backward stability, we split the responsibility across the training procedure: initialization is used to preserve forward activations, while gradient normalization before the optimizer step is used to control the scale of weight gradients.

\subsection{Why Forward Normalization is Problematic in PINNs}

PINNs are a representative case where scale propagation becomes especially delicate. A PINN represents a continuous field, for example $u_\theta(x,t)$, and is trained not only from data, but also by penalizing the residual of a differential equation \citep{raissi2019physics}. Given collocation points $\{(x_i,t_i)\}_{i=1}^{N_f}$, a typical physics loss is
\begin{equation}
    \mathcal{L}_{\mathrm{PDE}}(\theta)
    =
    \frac{1}{N_f}
    \sum_{i=1}^{N_f}
    \left|
    \mathcal{N}[u_\theta](x_i,t_i)
    \right|^2,
\end{equation}
where $\mathcal{N}$ is a differential operator involving derivatives of the network output with respect to its inputs. This changes the gradients received by the optimizer. For the physics loss,
\begin{equation}
    \nabla_\theta \mathcal{L}_{\mathrm{PDE}}
    =
    \frac{2}{N_f}
    \sum_{i=1}^{N_f}
    \mathcal{N}[u_\theta](x_i,t_i)
    \nabla_\theta \mathcal{N}[u_\theta](x_i,t_i).
\end{equation}
Thus, if $\mathcal{N}$ contains terms such as $u_x$, $u_{xx}$, or $\Delta u$, the gradient involves quantities such as $\nabla_\theta \partial_x u_\theta,     \nabla_\theta \partial_{xx} u_\theta,  \nabla_\theta \Delta u_\theta.$ These are parameter sensitivities of input derivatives, not merely sensitivities of the output itself. Consequently, even in constant-width PINNs where hidden layers satisfy \texttt{fan-in} $=$ \texttt{fan-out}, 
the physics-informed objective can induce gradient-scale imbalances across depth. 
The difficulty is therefore not only architectural; it is also introduced by the differential structure of the loss. Appendix~\ref{app:propagation} shows that, unlike in standard supervised networks where forward and backward variance propagation can be approximated by scalar recursions, PINN residuals involving input derivatives create coupled derivative-adjoint channels. 
This makes it unlikely that a single initialization rule can reliably preserve backward scale across all relevant channels.

Recent work has similarly shown that PINN training is often dominated by optimization pathologies rather than only by approximation capacity. Gradient-flow and NTK analyses identify imbalanced convergence rates across loss components and stiff training dynamics, while loss-landscape analyses connect these difficulties to ill-conditioning induced by differential operators \citep{wang2021understanding,wang2022when}. These observations are consistent with our motivation: differential residuals can create backward-scale imbalances that are not resolved by initialization alone.

Batch-dependent normalization is also problematic in this context. The output at a point should represent a local physical quantity, and its derivatives should be consistent with that local interpretation. BatchNorm can interfere with this structure because the prediction at one point may depend on other points in the same batch. This is particularly problematic when collocation points, boundary points, and initial-condition points are sampled from different distributions but are jointly used to define the physical objective.

This creates a depth-scaling bottleneck: the standard architectural normalization tools that make deep networks trainable are largely unavailable, while the physics-informed objective can still produce unstable gradient magnitudes across depth. We therefore seek a strategy that preserves the physical forward model unchanged and moves gradient-scale control to the optimization procedure.

\subsection{Architectural Mechanisms for Scale Stabilization}

Architectural mechanisms provide a widely used practical answer to the scale propagation problem. Normalization layers such as Batch Normalization, Layer Normalization, Group Normalization, and related methods re-scale intermediate representations during training \citep{ioffe2015batch,ba2016layer,wu2018group}. These layers can be interpreted as scale-resetting mechanisms inserted inside the network. By repeatedly normalizing hidden representations, they reduce the burden placed on initialization and make both forward activations and backward gradients easier to control.

Residual connections provide a complementary mechanism by creating shorter paths for signal and gradient propagation across depth \citep{he2016deep}. Rather than re-scaling representations, they improve trainability by allowing information and gradients to bypass long chains of transformations. In practice, modern deep architectures often combine residual pathways with normalization layers, making optimization substantially less sensitive to the precise initialization scale.

This is especially important in modern convolutional architectures, where Batch Normalization and residual connections are central components of stable deep training. With these architectural stabilizers, scale errors are repeatedly corrected or bypassed. Conversely, when such mechanisms are removed, the network again becomes much more exposed to the accumulation of scale errors across depth.

However, these architectural stabilizers do not solve the setting considered here. Batch-dependent normalization can alter the local physical interpretation of PINN outputs and derivatives, while residual connections do not directly control the scale of the weight gradients produced by differential operators. We therefore seek a mechanism that leaves the forward model unchanged, does not modify the physical residual, and acts only on the gradients passed to the optimizer.


\section{StableGrad}
\label{sec:stablegrad}

The previous discussion suggests a simple principle: forward-scale control and backward-scale control need not be imposed by the same mechanism. Initialization can be used to set a well-scaled forward pass at the beginning of training, while the scale of the gradients can be corrected directly before the optimizer update. StableGrad follows this principle.

Let the trainable parameters be partitioned into layerwise blocks,
\(\theta=(\theta^1,\ldots,\theta^L)\), and let
\(g^\ell=\nabla_{\theta^\ell}\mathcal{L}\) denote the gradient of block \(\ell\).
After the backward pass, StableGrad computes the empirical standard deviation
\(\sigma_\ell=\operatorname{std}(g^\ell)\) of each block gradient. It also computes a reference scale from the adjoint signal at the network output,
\(\sigma_{\mathrm{out}}=\operatorname{std}(\partial\mathcal{L}/\partial u_\theta)\).
The gradient passed to the optimizer is then
\begin{equation}
    \widetilde g^\ell
    =
    \frac{\sigma_{\mathrm{out}}}{\sigma_\ell+\varepsilon}
    g^\ell ,
    \qquad \ell=1,\ldots,L,
    \label{eq:stablegrad}
\end{equation}
where \(\varepsilon>0\) is a small numerical constant.

Thus, StableGrad transfers the scale of the signal that initiates the backward pass to all parameter blocks. 
The method acts after the loss and its derivatives have been computed, and before the optimizer update. 
The rescaled gradients \(\widetilde g^\ell\) are then passed to the optimizer in place of the raw gradients \(g^\ell\).

The choice of \(\sigma_{\mathrm{out}}\) is motivated by the role of the output adjoint as the source of the backward signal. 
StableGrad does not aim to preserve the global norm of the gradient. 
Instead, it enforces a layerwise notion of backward-scale consistency: all blocks are updated from gradients expressed at the same statistical scale as the output adjoint. 
This is precisely the scale that is propagated backwards through the network. 
Appendix~\ref{app:reference_scale} discusses this choice in more detail and describes alternative reference scales that can be used to separate layerwise balancing from changes in global gradient scale.

This makes StableGrad a natural companion to \texttt{fan-in} initialization. 
A \texttt{fan-in} scheme provides an activation-preserving starting point, setting forward magnitudes to propagate stably at initialization, while StableGrad dynamically controls the scale of layer-wise weight gradients throughout optimization. 
This separation avoids using initialization as a compromise between forward and backward propagation: forward activations are initialized in the regime targeted by \texttt{fan-in}, and gradient-scale imbalances are corrected by StableGrad after each backward pass, without introducing normalization layers into the forward model.

The distinction is particularly useful in PINNs. The physical residual is evaluated from the network output and its input derivatives, such as \(u_x\), \(u_{xx}\), or \(u_t\). StableGrad leaves this forward map untouched: the network \(u_\theta\), its derivatives, and the physics residual are computed exactly as in the original model. Only the gradient delivered to the optimizer is rescaled. In this sense, StableGrad provides backward stabilization while preserving the physical interpretation of the forward model.

Figure~\ref{fig:gradient_flow} illustrates the different roles of initialization, BatchNorm, LayerNorm, and StableGrad in a controlled MLP. With vanilla initialization, \texttt{fan-in} preserves forward activation scales but leaves layer-wise weight gradients highly imbalanced, whereas \texttt{fan-out} improves backward signal scaling at the cost of increasing activation scales. BatchNorm and LayerNorm both reduce forward-scale sensitivity by repeatedly normalizing intermediate representations, making the \texttt{fan-in} and \texttt{fan-out} cases more similar. However, the resulting weight-gradient scales remain uneven across layers. StableGrad, instead, keeps the \texttt{fan-in} forward computation unchanged and acts only after backpropagation, equalizing the layer-wise weight-gradient scales passed to the optimizer.

\begin{figure}[htb]
    \centering
    \begin{subfigure}{0.5\linewidth}
        \caption{MLP with fan-in}
        \includegraphics[width=1\linewidth,trim={0cm 0cm 0.84cm 0.6cm},clip]{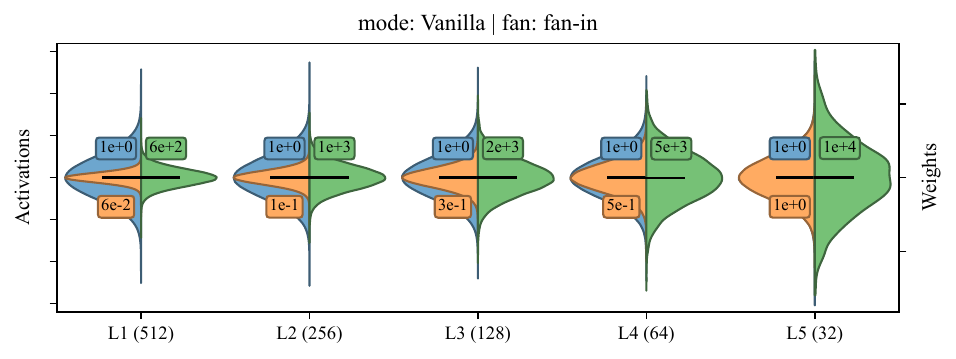}
    \end{subfigure}%
    \begin{subfigure}[b]{0.5\linewidth}
        \caption{MLP with fan-out}
        \includegraphics[width=1\linewidth,trim={0.84cm 0cm 0cm 0.6cm},clip]{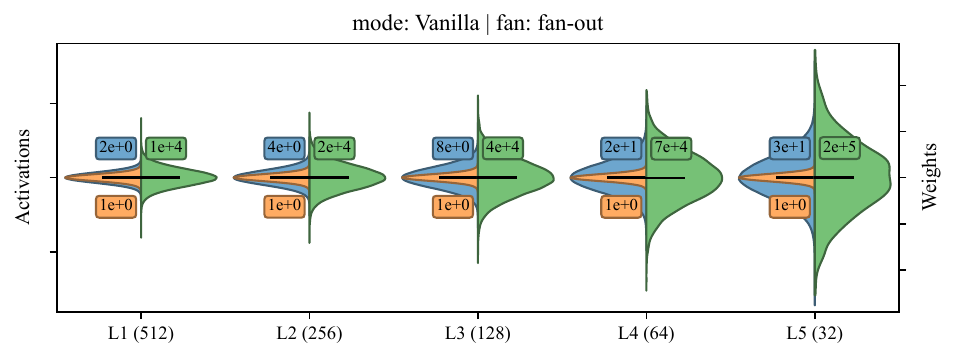}
    \end{subfigure}\\
    \begin{subfigure}{0.5\linewidth}
        \caption{MLP with BatchNorm and fan-in}
        \includegraphics[width=1\linewidth,trim={0cm 0cm 0.84cm 0.6cm},clip]{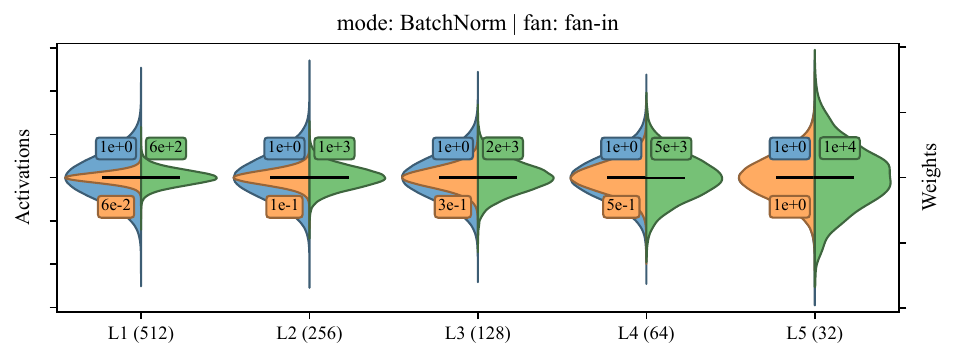}
    \end{subfigure}%
    \begin{subfigure}{0.5\linewidth}
        \caption{MLP with BatchNorm and fan-out}
        \includegraphics[width=1\linewidth,trim={0.84cm 0cm 0cm 0.6cm},clip]{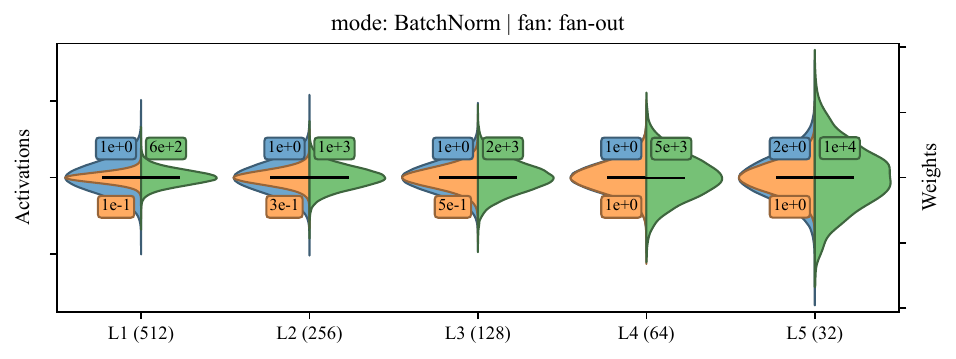}
    \end{subfigure}\\
    \begin{subfigure}{0.5\linewidth}
        \caption{MLP with LayerNorm and fan-in}
        \includegraphics[width=1\linewidth,trim={0cm 0cm 0.84cm 0.6cm},clip]{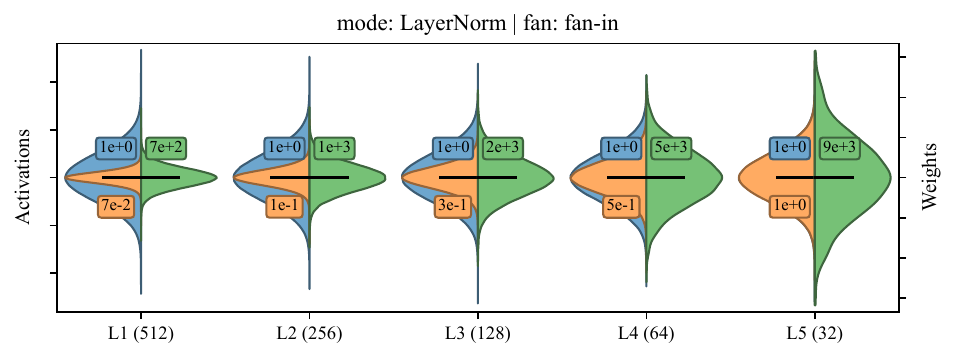}
    \end{subfigure}%
    \begin{subfigure}{0.5\linewidth}
        \caption{MLP with LayerNorm and fan-out}
        \includegraphics[width=1\linewidth,trim={0.84cm 0cm 0cm 0.6cm},clip]{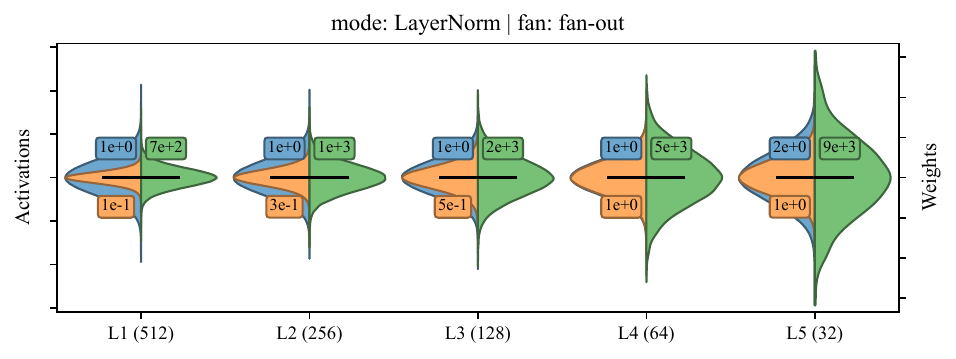}
    \end{subfigure}\\
    \begin{subfigure}{0.49\linewidth}
        \caption{MLP with fan-in and StableGrad}
        \includegraphics[width=1\linewidth,trim={0cm 0cm 0.72cm 0.6cm},clip]{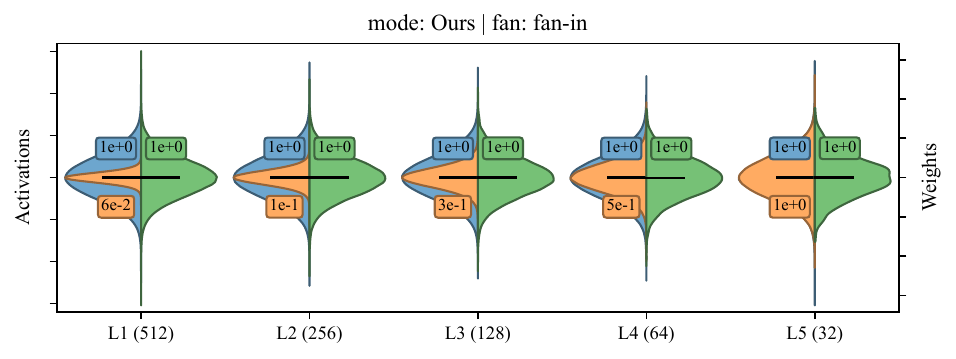}
    \end{subfigure}%
    \begin{subfigure}[t]{0.49\linewidth}
        \centering
        \includegraphics[width=0.5\linewidth,trim={1cm 0.5cm 1cm 1cm},clip]{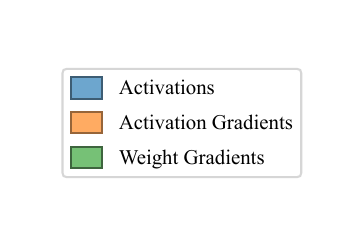}
    \end{subfigure}
    \caption{
    Forward and backward scale diagnostics across depth. StableGrad preserves the forward computation and equalizes the scale of layerwise weight gradients before the optimizer update.
    }
    \label{fig:gradient_flow}
\end{figure}

\section{Effective Training Dynamics}
\label{sec:effective_dynamics}

We now analyze the local effect of StableGrad on the training dynamics. The full derivations are given in Appendices~\ref{app:effective_kernel} and~\ref{app:local_decrease}.

Consider a weighted least-squares objective
\[
    \mathcal{L}(\theta)=\frac12\|r(\theta)\|^2,
\]
where \(r(\theta)\in\mathbb{R}^N\) is the vector of weighted residuals. This notation covers supervised residuals as well as the data, boundary, initial-condition, and PDE residuals used in PINNs. Let \(J=\partial r/\partial\theta\) be the residual Jacobian. Under the local linearization \(r(\theta+\Delta\theta)\simeq r(\theta)+J\Delta\theta\), a standard gradient step gives the residual dynamics
\[
    r^+ \simeq r-\eta JJ^\top r.
\]
The matrix \(K=JJ^\top\) is the empirical neural tangent kernel governing this local evolution.

StableGrad modifies the gradient before the optimizer step. If we define the block-diagonal matrix
\[
    P=\operatorname{diag}(\alpha_1I_1,\ldots,\alpha_LI_L),
    \qquad
    \alpha_\ell=\frac{\sigma_{\mathrm{out}}}{\sigma_\ell+\varepsilon},
\]
then the idealized StableGrad step is \(\Delta\theta=-\eta PJ^\top r\). The corresponding residual dynamics are governed by the effective kernel
\begin{equation}
    K_{\mathrm{SG}} = JPJ^\top .
    \label{eq:ksg}
\end{equation}
Writing the Jacobian by parameter blocks as \(J=[J_1,\ldots,J_L]\), this kernel decomposes as
\[
    K_{\mathrm{SG}}
    =
    \sum_{\ell=1}^L
    \alpha_\ell J_\ell J_\ell^\top .
\]
StableGrad therefore reweights the contribution of each layer to the functional training dynamics.

To measure how strongly a kernel acts on the current residual, we use the residual Rayleigh quotient
\[
    \rho_A(r)=\frac{r^\top A r}{\|r\|^2}.
\]
We denote by \(\rho=\rho_K(r)\) the standard value and by
\(\rho_{\mathrm{SG}}=\rho_{K_{\mathrm{SG}}}(r)\) the StableGrad value. Since \(g^\ell=J_\ell^\top r\), the change induced by StableGrad is explicit:
\[
    \rho_{\mathrm{SG}}-\rho
    =
    \frac{
    \sum_{\ell=1}^L
    (\alpha_\ell-1)\|g^\ell\|^2
    }{\|r\|^2}.
\]
Hence, \emph{StableGrad increases the kernel action on the current residual whenever the amplified blocks carry enough gradient energy to dominate the blocks that are downscaled}.

The following result connects this quantity with the decrease of the linearized loss.

\begin{theorem}[Local decrease under StableGrad]
\label{thm:local_decrease}
Let \(K=JJ^\top\) and \(K_{\mathrm{SG}}=JPJ^\top\). Consider the linearized residual dynamics generated by the standard step and by the StableGrad step. If
\begin{equation}
    \eta\lambda_{\max}(K_{\mathrm{SG}})<2
    \quad\text{and}\quad
    \rho_{\mathrm{SG}}
    \left(
        1-\frac{\eta}{2}\lambda_{\max}(K_{\mathrm{SG}})
    \right)
    >
    \rho,
    \label{eq:stablegrad_local_condition}
\end{equation}
then the StableGrad linearized step produces a larger decrease of
\(\frac12\|r\|^2\) than the standard linearized gradient step.
\end{theorem}

The proof is given in Appendix~\ref{app:local_decrease}. The theorem separates the two quantities that matter locally. The term \(\rho_{\mathrm{SG}}\) measures the useful action of the effective kernel on the current residual, while \(\eta\lambda_{\max}(K_{\mathrm{SG}})\) controls the stability of the local step. Thus, StableGrad improves the local linearized decrease when it increases the action of the training kernel on the residual and keeps the step in a stable regime.

This result also gives a direct interpretation of the method. StableGrad is beneficial when the backward scale imbalance suppresses layers that still contain useful directions for reducing the current residual. By rescaling those blocks, the method increases their contribution to \(K_{\mathrm{SG}}\). The theorem then states when this increased contribution translates into a larger local decrease of the loss.

Appendix~\ref{app:effective_diagnostics} provides a controlled Burgers PINN diagnostic that directly measures the quantities appearing in Theorem~\ref{thm:local_decrease}. 
In that experiment, StableGrad reduces the validation loss much faster than AdamW, reaching by epoch 1000 a lower loss than AdamW attains at the end of training for 5000 epochs. 
The diagnostic quantities explain this behavior: StableGrad flattens the layer-wise gradient scales, keeps the stability factor \(\eta\lambda_{\max}(K_{\mathrm{SG}})\) far below the instability threshold, and satisfies the theorem margin during the phases where the loss decreases most effectively. 
The few checkpoints where the margin becomes negative coincide with the transient plateau in the loss curve, where the theorem no longer predicts an improved local decrease. 
Thus, the empirical dynamics match the theoretical picture: StableGrad improves the effective backward dynamics while leaving the forward physical model unchanged.

A natural question is whether StableGrad mainly acts as an implicit learning-rate schedule, since rescaling gradients can change the effective step size seen by the optimizer. 
Appendix~\ref{app:lr_scheduler_control} addresses this directly with a control experiment in which AdamW is equipped with a piecewise learning-rate multiplier chosen from the observed spectral ratio \(\lambda_{\max}(K_{\mathrm{SG}})/\lambda_{\max}(K)\). 
This boosted AdamW baseline improves over standard AdamW, but it does not reproduce StableGrad: the residual loss remains higher and the layer-wise update distribution stays much more concentrated. 
Thus, the effect of StableGrad is not reducible to a global learning-rate increase; it changes the effective geometry of the update.

\section{Evaluation}\label{sec:evaluation}

We evaluate StableGrad along two complementary axes. First, we use deep CNNs as a diagnostic setting to test whether optimizer-level gradient normalization can preserve trainability when BatchNorm is removed from architectures that normally rely on it. This setting allows us to observe whether training remains numerically stable without architectural normalization. Second, we evaluate deep PINNs, where BatchNorm is not a suitable stabilizer, and test whether forward-stable initialization combined with StableGrad improves depth scaling and solution accuracy.

\subsection{CNNs}\label{subsec:CNNs}

We first evaluate StableGrad on image classification models where BatchNorm is a standard component of the architecture. 
We consider EfficientNetV2-S \citep{efficientnetv2} on CIFAR-100 \citep{krizhevsky2009learning} and ResNet-50 \citep{he2016deep} on ImageNet-1k \citep{deng2009imagenet}. 
For each model, we compare three variants: the default architecture with BatchNorm (Vanilla), the same architecture with BatchNorm removed (Without BN), and the BatchNorm-free architecture trained with StableGrad. 
All variants are trained under the same protocol except for the presence of BatchNorm and the use of gradient normalization; Appendix~\ref{app:std_wo_bn} further analyzes the BatchNorm-free failure mode by tracking how activation scales evolve until training collapse. 
To check that the effect is not obtained by a trivial scale-removal rule, Appendix~\ref{app:sign_gradient_comparison} also compares against sign-based gradient preprocessing, which homogenizes gradients more aggressively but fails to train in the same EfficientNet setting. 
Further implementation and training details are provided in Appendix~\ref{app:eval_details}.

\begin{figure}[tbh]
    \centering
    \begin{subfigure}{0.48\linewidth}
        \caption{EfficientNetV2-S on CIFAR-100}
        \includegraphics[width=1\linewidth,trim={0cm 0cm 0cm 0.65cm},clip]{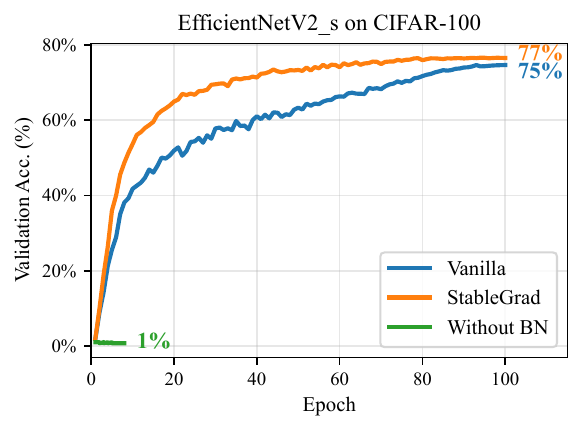}
    \end{subfigure}%
    \begin{subfigure}{0.48\linewidth}
        \caption{ResNet-50 on ImageNet-1k}
        \includegraphics[width=1\linewidth,trim={0cm 0cm 0cm 0.65cm},clip]{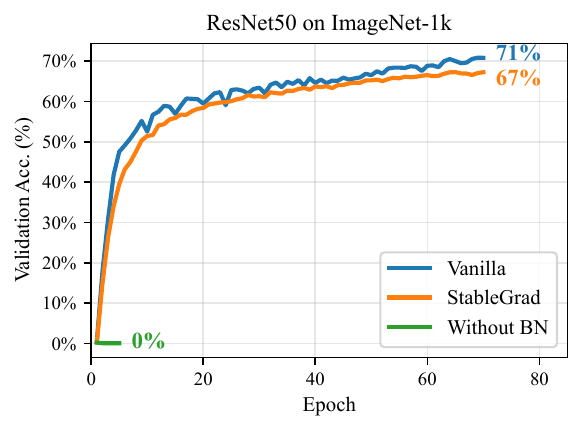}
    \end{subfigure}\\
    \caption{
    Training dynamics of deep CNNs with and without BatchNorm. 
    The default architectures with BatchNorm and the BatchNorm-free variants trained with StableGrad remain trainable, while removing BatchNorm without gradient control leads to training collapse after a few iterations due to forward-pass overflow. Differences between the BatchNorm and StableGrad curves reflect different effective learning-rate dynamics and should not be interpreted as a direct training-speed comparison.
    }
    \label{fig:cnn_training}
\end{figure}

Figure~\ref{fig:cnn_training} shows the diagnostic role of the CNN experiments. 
In both EfficientNetV2-S and ResNet-50, simply removing BatchNorm makes training collapse almost immediately, confirming that these architectures rely strongly on architectural normalization for stable optimization. 
StableGrad prevents this collapse without reintroducing any forward normalization. 
On EfficientNetV2-S, the BatchNorm-free model trained with StableGrad converges faster and reaches a higher validation accuracy than the BatchNorm baseline, exceeding \(77\%\) compared with roughly \(75\%\) for the default model. 
On ResNet-50, StableGrad remains close to the BatchNorm baseline, reaching about \(67\%\) validation accuracy while the default BatchNorm model reaches about \(71\%\). 
The relevant point is not a one-to-one speed comparison, since BatchNorm and StableGrad induce different effective training dynamics, but rather that optimizer-level gradient normalization makes otherwise non-trainable BatchNorm-free CNNs train stably and reach competitive accuracy while leaving the forward architecture unnormalized.

\subsection{PINNs}

StableGrad is evaluated in the target setting of deep PINNs. The experiments consider three PDE benchmarks commonly used in physics-informed learning: Burgers' equation with viscosity \(\nu=10^{-4}\), Poisson's equation, and a high-frequency \(k=10\pi\) Helmholtz problem \citep{raissi2019physics,hao2024pinnacle}. For each benchmark, fully connected PINNs with depths \(6\) and \(12\) are trained. Burgers and Poisson use tanh activations, while Helmholtz uses SiLU activations together with Fourier feature inputs. All PDE constraints are enforced softly through the optimization objective, by penalizing the PDE residual together with the corresponding boundary-condition and initial-condition losses.

The comparison is between AdamW \citep{loshchilov2018decoupled} and AdamW equipped with StableGrad. The AdamW baseline is trained for \(50{,}000\) optimization steps. In the StableGrad setting, StableGrad is applied during the first \(25{,}000\) steps. Once the layer-wise gradient scales have stabilized, training continues with standard AdamW for an additional \(25{,}000\) fine-tuning steps. Further benchmark, architecture, and hyperparameter details are provided in Appendix~\ref{app:eval_details}.

Table~\ref{tab:pinn_eval} reports validation metrics for each benchmark: relative \(L_2\) error, PDE residual loss, boundary-condition loss, and, for Burgers' equation, initial-condition loss. All reported PINN values are means over three independent runs with different random seeds. Across all benchmark--depth configurations, AdamW+StableGrad outperforms the corresponding Vanilla AdamW baseline in nearly all reported metrics. This holds for both shallow and deeper PINNs, and for both solution accuracy and physical consistency terms. In some cases the improvement is large, while in others the gains are more modest; nevertheless, the direction is consistent across all reported comparisons.

\begin{table}[ht]
\centering
\caption{Evaluation comparison between vanilla and StableGrad with increasing PINN depth in validation points.}
\label{tab:pinn_eval}
\begin{tabular*}{\columnwidth}{@{\extracolsep{\fill}}ccccccc@{}}
\toprule
\textbf{Experiment} &
\textbf{Depth} &
\textbf{Mode} &
\textbf{Validation L2} &
\textbf{PDE loss} &
\textbf{IC loss} &
\textbf{BC loss} \\
\midrule[1.2pt]

\multirow[c]{4}{*}{Burgers}
& \multirow[c]{2}{*}{6} & Vanilla    & 1.3e-1 & 1.0e-3 & 1.2e-5 & 4.6e-8 \\
&                       & StableGrad & \textbf{8.5e-2} & \textbf{6.3e-4} & \textbf{7.6e-6} & \textbf{1.8e-8} \\
\cmidrule{2-7}
& \multirow[c]{2}{*}{12} & Vanilla    & 1.6e-1 & 6.8e-3 & 8.2e-5 & 4.7e-8 \\
&                        & StableGrad & \textbf{1.1e-2} & \textbf{9.0e-4} & \textbf{8.7e-6} & \textbf{3.1e-8} \\

\midrule[1.2pt]

\multirow[c]{4}{*}{Poisson}
& \multirow[c]{2}{*}{6} & Vanilla    & 6.5e-5 & 8.6e-6 & -- & 7.9e-9 \\
&                       & StableGrad & \textbf{2.6e-5} & \textbf{1.8e-6} & -- & \textbf{1.7e-9} \\
\cmidrule{2-7}
& \multirow[c]{2}{*}{12} & Vanilla    & 5.9e-5 & 4.0e-6 & -- & 7.7e-9 \\
&                        & StableGrad & \textbf{3.8e-5} & \textbf{3.1e-6} & -- & \textbf{3.5e-9} \\

\midrule[1.2pt]

\multirow[c]{4}{*}{Helmholtz}
& \multirow[c]{2}{*}{6} & Vanilla    & 4.6e-3 & 6.0e-6 & -- & 1.9e-8 \\
&                       & StableGrad & \textbf{3.1e-3} & 6.0e-6 & -- & \textbf{1.7e-8} \\
\cmidrule{2-7}
& \multirow[c]{2}{*}{12} & Vanilla    & 6.5e-3 & 1.0e-5 & -- & 2.5e-8 \\
&                        & StableGrad & \textbf{2.3e-3} & \textbf{6.6e-6} & -- & \textbf{1.7e-8} \\

\bottomrule
\end{tabular*}
\end{table}

These results indicate that StableGrad improves the matched-depth optimization of PINNs rather than merely shifting error between different components of the objective. The method reduces the validation \(L_2\) error while also lowering the PDE, boundary-condition, and, when applicable, initial-condition losses. Thus, the improvements are not obtained by fitting the solution field at the expense of the physical constraints, but by producing solutions that are both more accurate and more consistent with the imposed equations.

The depth comparison further shows that simply increasing PINN capacity is not sufficient under standard optimization. Vanilla AdamW does not consistently benefit from moving from depth \(6\) to depth \(12\), and in some benchmarks the deeper baseline remains similar to, or worse than, its shallower counterpart. StableGrad mitigates this depth-scaling difficulty: at each evaluated depth, it produces better validation metrics than Vanilla, and the deeper StableGrad models remain trainable and competitive across all three PDE benchmarks. The Poisson case illustrates an important nuance. Since this benchmark is comparatively simple, the depth-\(12\) StableGrad model obtains lower training losses but slightly worse validation metrics than its depth-\(6\) counterpart, suggesting that the additional capacity may lead to overfitting rather than improved generalization. Thus, StableGrad improves the optimization of deeper PINNs, but additional depth is not automatically beneficial when the target problem does not require the extra capacity. These results support the main claim of the paper: when architectural normalization is unavailable or undesirable, controlling layer-wise gradient scale before the optimizer step provides a practical mechanism for improving the trainability and accuracy of deep PINNs.

\section{Conclusion}
Depth turns scale propagation into a bidirectional problem: activations must remain well scaled in the forward pass, while adjoint signals and weight gradients must remain useful in the backward pass. Classical initialization schemes address this tension only at the beginning of training and necessarily make a trade-off between forward- and backward-preserving fan modes. Architectural normalization layers, especially BatchNorm, alleviate this burden by repeatedly resetting intermediate scales, but they also introduce forward-pass dependencies that are undesirable in settings such as PINNs, where the network output and its input derivatives define a physical field and its residual.

We introduced StableGrad, an optimizer-level mechanism focused on stabilizing backward scale. It keeps layer-wise gradient magnitudes controlled during training by normalizing weight gradients after backpropagation and before the optimizer update. Since this operation is applied only at the update stage, StableGrad leaves the forward model, its derivatives, and the physical residual unchanged, providing backward scale control without inserting normalization layers into the architecture. It therefore pairs naturally with activation-aware fan-in initializations that preserve forward scale.

The experiments support this separation. In BatchNorm-free CNNs, StableGrad prevents the optimization collapse that occurs when BatchNorm is removed from architectures that normally rely on it. In PINNs, where BatchNorm is not an appropriate stabilizer, StableGrad improves matched-depth accuracy across Burgers, Poisson, and Helmholtz benchmarks and mitigates the optimization difficulties observed in deeper models. These results suggest that optimizer-level control of gradient scale can be a practical alternative when architectural normalization is unavailable or undesirable.

StableGrad also has important limitations. The method normalizes the weight gradients after the backward pass, but it does not prevent the raw adjoint signals from becoming unstable while they are being propagated through the network. Similarly, although the initialization is designed to start from a forward-stable regime, StableGrad does not by itself guarantee that activation magnitudes will remain stable throughout training. Thus, the method should be viewed as a backward-scale correction mechanism rather than a complete solution to all forward and backward stability problems.

Future work should study how to guarantee forward-pass stability during training while preserving the physical interpretation required by PINNs. Another promising direction is to combine layer-wise gradient normalization with automatic learning-rate adaptation, for example with techniques related to \cite{defazio2024the}, in order to control the effective step size and reduce dependence on a learning-rate scheduler. Finally, the evaluation should be extended beyond the PINN and CNN settings considered here, including SIREN-style models, implicit neural representations, and other architectures where forward normalization is difficult to use but depth remains essential.

\begin{ack}
TBD
\end{ack}

\bibliographystyle{apalike}
\bibliography{biblio}


\appendix

\section{Additional Theory}
\label{app:additional_theory}

This section collects the theoretical details that support the main analysis of StableGrad. 
Appendix~\ref{app:propagation} explains why standard scalar variance-propagation arguments become insufficient for PINNs with differential residuals, while Appendix~\ref{app:reference_scale} discusses alternative global reference scales for the StableGrad normalization. 
Appendices~\ref{app:effective_kernel} and~\ref{app:local_decrease} then derive the StableGrad effective kernel, the associated Rayleigh-quotient identity, and the proof of the local decrease theorem used in the main text.

\subsection{Signal Propagation with Differential Residuals}
\label{app:propagation}

This appendix summarizes why the usual variance-preservation argument for standard multilayer perceptrons becomes more involved in PINNs.

Consider the feedforward recursion \(z_\ell=W_\ell h_\ell\) and \(h_{\ell+1}=\phi(z_\ell)\), with independently initialized zero-mean weights of variance \(\sigma_W^2\). Under the standard independence and equal-scale assumptions, the forward variance is approximately
\[
    \operatorname{Var}(h_{\ell+1})
    \approx
    n_\ell\sigma_W^2\operatorname{Var}(h_\ell).
\]
If \(\delta_\ell=\partial\mathcal{L}/\partial h_\ell\) denotes the backpropagated adjoint, then the backward recursion gives
\[
    \operatorname{Var}(\delta_\ell)
    \approx
    n_{\ell+1}\sigma_W^2
    \mathbb{E}[\phi'(z_\ell)^2]
    \operatorname{Var}(\delta_{\ell+1}).
\]
Thus, in a standard MLP, the forward and backward scales can be approximated by scalar variance recursions. This is precisely the setting in which fan-in, fan-out, and fan-in/fan-out initialization rules are naturally derived.

PINNs add a different difficulty. The loss may depend on input derivatives of the network output. To see the effect, consider the first derivative channel \(p_\ell=\partial_x h_\ell\). Differentiating the layer recursion gives \(p_{\ell+1}=\phi'(z_\ell)\odot W_\ell p_\ell\). Now suppose that the loss depends on both \(h_{\ell+1}\) and \(p_{\ell+1}\). Define the adjoints \(a_{\ell+1}=\partial\mathcal{L}/\partial h_{\ell+1}\) and \(b_{\ell+1}=\partial\mathcal{L}/\partial p_{\ell+1}\). Backpropagation gives
\[
    b_\ell
    =
    W_\ell^\top
    \left(
        \phi'(z_\ell)\odot b_{\ell+1}
    \right),
\]
while the activation-channel adjoint satisfies
\[
    a_\ell
    =
    W_\ell^\top
    \left(
        \phi'(z_\ell)\odot a_{\ell+1}
        +
        \phi''(z_\ell)\odot (W_\ell p_\ell)\odot b_{\ell+1}
    \right).
\]
The second term has no analogue in a standard supervised MLP. It couples the derivative-channel adjoint \(b_{\ell+1}\) into the activation-channel adjoint \(a_\ell\), and depends on \(\phi''\), on the derivative state \(p_\ell\), and on the weights.

Higher-order differential residuals introduce higher derivative channels and higher derivatives of the activation. Consequently, the PINN backward pass is not governed by a single scalar variance recursion. A single initialization variance cannot, in general, simultaneously preserve the scale of all derivative-dependent backward channels. This motivates a dynamic correction of gradient scale during training.

\subsection{Reference Scale and Useful Variants}
\label{app:reference_scale}

StableGrad uses the output-adjoint scale, setting \(\alpha_\ell=\sigma_{\mathrm{out}}/(\sigma_\ell+\varepsilon)\). This choice follows the backward-scale principle: the scale of the signal that initiates the backward pass is transferred to every parameter block.

A more general family is obtained by writing \(\alpha_\ell(c)=c/(\sigma_\ell+\varepsilon)\), where \(c>0\) is a global reference scale. Different choices of \(c\) preserve different quantities and can be useful for ablations.

A norm-preserving reference scale is obtained by imposing \(\|\widetilde g\|=\|g\|\). Since \(\widetilde g^\ell=cg^\ell/(\sigma_\ell+\varepsilon)\), this gives
\[
    c_{\mathrm{norm}}
    =
    \left(
    \frac{
    \sum_{\ell=1}^L \|g^\ell\|^2
    }{
    \sum_{\ell=1}^L
    \|g^\ell\|^2/(\sigma_\ell+\varepsilon)^2
    }
    \right)^{1/2}.
\]
An inner-product preserving reference scale is obtained by imposing \(g^\top\widetilde g=g^\top g\). This gives
\[
    c_{\mathrm{ip}}
    =
    \frac{
    \sum_{\ell=1}^L \|g^\ell\|^2
    }{
    \sum_{\ell=1}^L
    \|g^\ell\|^2/(\sigma_\ell+\varepsilon)
    }.
\]
These alternatives separate the effect of layerwise balancing from changes in global scale. The StableGrad choice \(c=\sigma_{\mathrm{out}}\) follows a different criterion: it anchors every block gradient to the natural scale of the backward signal at the output.

\subsection{Effective Kernel Induced by StableGrad}
\label{app:effective_kernel}

We derive the effective kernel used in Section~\ref{sec:effective_dynamics} and prove the identity for the residual Rayleigh quotient.

\begin{proposition}[StableGrad effective kernel]
\label{prop:effective_kernel}
Let \(\mathcal{L}(\theta)=\frac12\|r(\theta)\|^2\), let \(J=\partial r/\partial\theta\), and partition the parameters as \(\theta=(\theta^1,\ldots,\theta^L)\). Write the Jacobian by blocks as \(J=[J_1,\ldots,J_L]\). If StableGrad rescales each block gradient by \(\alpha_\ell\), then the linearized residual dynamics are governed by
\[
    K_{\mathrm{SG}}=JPJ^\top
    =
    \sum_{\ell=1}^L \alpha_\ell J_\ell J_\ell^\top,
\]
where \(P=\operatorname{diag}(\alpha_1I_1,\ldots,\alpha_LI_L)\).
\end{proposition}

\begin{proof}
Since \(\mathcal{L}(\theta)=\frac12\|r(\theta)\|^2\), the gradient is \(g=J^\top r\). The block gradient is therefore \(g^\ell=J_\ell^\top r\). StableGrad rescales the gradient as \(\widetilde g=Pg\), with \(P=\operatorname{diag}(\alpha_1I_1,\ldots,\alpha_LI_L)\).

Using the local approximation \(r(\theta+\Delta\theta)\simeq r(\theta)+J\Delta\theta\), the idealized StableGrad step \(\Delta\theta=-\eta PJ^\top r\) gives \(r^+\simeq r-\eta JPJ^\top r\). Hence the effective kernel is \(K_{\mathrm{SG}}=JPJ^\top\). Finally, because \(J=[J_1,\ldots,J_L]\) and \(P\) is block diagonal, \(JPJ^\top=\sum_{\ell=1}^L \alpha_\ell J_\ell J_\ell^\top\).
\end{proof}

We now prove the identity used in the main text.

\begin{proposition}[Residual Rayleigh quotient under StableGrad]
\label{prop:rho_identity_appendix}
Let \(K=JJ^\top\) and \(K_{\mathrm{SG}}=JPJ^\top\). Define
\[
    \rho=\frac{r^\top K r}{\|r\|^2},
    \qquad
    \rho_{\mathrm{SG}}=\frac{r^\top K_{\mathrm{SG}} r}{\|r\|^2}.
\]
Then
\[
    \rho_{\mathrm{SG}}-\rho
    =
    \frac{
    \sum_{\ell=1}^L(\alpha_\ell-1)\|g^\ell\|^2
    }{\|r\|^2},
\]
where \(g^\ell=J_\ell^\top r\).
\end{proposition}

\begin{proof}
Using the block decomposition of \(K_{\mathrm{SG}}\), we have
\[
    r^\top K_{\mathrm{SG}}r
    =
    \sum_{\ell=1}^L
    \alpha_\ell r^\top J_\ell J_\ell^\top r
    =
    \sum_{\ell=1}^L
    \alpha_\ell \|J_\ell^\top r\|^2.
\]
Since \(g^\ell=J_\ell^\top r\), this is \(\sum_{\ell=1}^L\alpha_\ell\|g^\ell\|^2\). Similarly, \(r^\top K r=\sum_{\ell=1}^L\|g^\ell\|^2\). Subtracting both expressions and dividing by \(\|r\|^2\) gives the claim.
\end{proof}

\subsection{Proof of the Local Decrease Theorem}
\label{app:local_decrease}

We prove Theorem \ref{thm:local_decrease} stated in Section~\ref{sec:effective_dynamics}. The result is local: it applies to the linearized residual dynamics around the current parameters.

\setcounter{theorem}{0}
\begin{theorem}[Local decrease under StableGrad]
\label{thm:local_decrease_appendix}
Let \(K=JJ^\top\) and \(K_{\mathrm{SG}}=JPJ^\top\). Consider the standard linearized step \(r^+_{\mathrm{std}}=r-\eta Kr\) and the StableGrad linearized step \(r^+_{\mathrm{SG}}=r-\eta K_{\mathrm{SG}}r\). Let
\[
    \rho=\frac{r^\top K r}{\|r\|^2},
    \qquad
    \rho_{\mathrm{SG}}=\frac{r^\top K_{\mathrm{SG}}r}{\|r\|^2}.
\]
If \(\eta\lambda_{\max}(K_{\mathrm{SG}})<2\) and
\[
    \rho_{\mathrm{SG}}
    \left(
        1-\frac{\eta}{2}\lambda_{\max}(K_{\mathrm{SG}})
    \right)
    >
    \rho,
\]
then the StableGrad linearized step produces a larger decrease of \(\frac12\|r\|^2\) than the standard linearized gradient step.
\end{theorem}

\begin{proof}
For any positive semidefinite matrix \(A\), the linearized step \(r^+=r-\eta Ar\) changes the quadratic loss by
\[
    \frac12\|r\|^2-\frac12\|r^+\|^2
    =
    \eta r^\top A r
    -
    \frac{\eta^2}{2}r^\top A^2r.
\]
Applying this identity with \(A=K_{\mathrm{SG}}\), the StableGrad decrease is
\[
    \Delta_{\mathrm{SG}}
    =
    \eta r^\top K_{\mathrm{SG}}r
    -
    \frac{\eta^2}{2}r^\top K_{\mathrm{SG}}^2r.
\]
Since \(K_{\mathrm{SG}}\) is positive semidefinite, \(r^\top K_{\mathrm{SG}}^2r\leq \lambda_{\max}(K_{\mathrm{SG}})r^\top K_{\mathrm{SG}}r\). Therefore,
\[
    \Delta_{\mathrm{SG}}
    \geq
    \eta r^\top K_{\mathrm{SG}}r
    \left(
        1-\frac{\eta}{2}\lambda_{\max}(K_{\mathrm{SG}})
    \right).
\]
Using \(r^\top K_{\mathrm{SG}}r=\|r\|^2\rho_{\mathrm{SG}}\), this becomes
\[
    \Delta_{\mathrm{SG}}
    \geq
    \eta\|r\|^2\rho_{\mathrm{SG}}
    \left(
        1-\frac{\eta}{2}\lambda_{\max}(K_{\mathrm{SG}})
    \right).
\]

For the standard step, the corresponding decrease is
\[
    \Delta_{\mathrm{std}}
    =
    \eta r^\top Kr
    -
    \frac{\eta^2}{2}r^\top K^2r.
\]
The second term is nonnegative, so \(\Delta_{\mathrm{std}}\leq \eta r^\top Kr=\eta\|r\|^2\rho\). Hence, the condition
\[
    \rho_{\mathrm{SG}}
    \left(
        1-\frac{\eta}{2}\lambda_{\max}(K_{\mathrm{SG}})
    \right)
    >
    \rho
\]
implies \(\Delta_{\mathrm{SG}}>\Delta_{\mathrm{std}}\). The additional condition \(\eta\lambda_{\max}(K_{\mathrm{SG}})<2\) ensures that the stability factor is positive.
\end{proof}

\section{Connecting Theory and Practice}
\label{app:theory_practice}

This section connects the local effective-dynamics analysis with controlled empirical diagnostics. 
Appendix~\ref{app:effective_diagnostics} reports a diagnostic Burgers run in which the quantities appearing in Theorem~\ref{thm:local_decrease} are measured directly, including the stability factor, theorem margin, residual linearization error, and layer-wise gradient-scale imbalance. 
Appendix~\ref{app:lr_scheduler_control} then tests whether the observed improvement can be explained purely by a global learning-rate increase, using a spectral scheduler control matched to the StableGrad effective-kernel scale.

\subsection{Controlled Effective-Dynamics Diagnostics}
\label{app:effective_diagnostics}

The main experiments in Section~\ref{sec:evaluation} evaluate the final performance of the method.
Here we use a controlled diagnostic run to test whether the local quantities appearing in
Theorem~\ref{thm:local_decrease} behave in practice as predicted by the analysis in
Section~\ref{sec:effective_dynamics}.

The diagnostic problem is a three-dimensional viscous Burgers PINN on the periodic domain
\([-1,1]^3\), with \(t\in[0,1]\) and viscosity \(\nu=0.05\). The reference solution is an exact
periodic solution constructed through the Cole--Hopf transformation. The network is a tanh MLP
mapping \((t,x,y,z)\) to \((u,v,w)\), using a rescaled time coordinate together with periodic
Fourier features in the spatial variables. It has four hidden layers of width 96 and 30,723
trainable parameters. We compare AdamW against AdamW equipped with StableGrad, using the same
learning rate \(\eta=10^{-3}\), no weight decay, and a fixed diagnostic batch evaluated every
500 epochs.

All diagnostics are computed from an explicit weighted residual vector \(r(\theta)\) such that
\[
    \mathcal{L}(\theta)=\frac12\|r(\theta)\|^2.
\]
For this Burgers PINN, \(r(\theta)\) concatenates the weighted PDE, initial-condition, and
periodic boundary residuals. In the StableGrad run, each parameter block gradient \(g^\ell\) is
rescaled with
\[
    \alpha_\ell
    =
    \frac{\sigma_{\mathrm{ref}}}
    {\operatorname{std}(g^\ell)+\varepsilon},
    \qquad
    \sigma_{\mathrm{ref}}=\operatorname{std}(r(\theta)).
\]

Figure~\ref{fig:effective_diagnostics_loss} shows the train and validation losses for the two
optimizers. The practical effect is immediate: StableGrad decreases the loss much faster than
AdamW and reaches, by epoch 1000, a validation loss already lower than the final validation loss
obtained by AdamW at epoch 5000. The StableGrad trajectory is not strictly monotone, however.
After the rapid initial descent, the loss enters a short transient regime around epochs 3000 and
3500, where the validation loss slightly worsens before improving again. This behavior is exactly
the regime in which the sufficient condition in Theorem~\ref{thm:local_decrease} ceases to hold,
as shown below.

\begin{figure}[t]
    \centering
    \includegraphics[width=\linewidth,trim={0cm 0cm 0cm 0.8cm},clip]{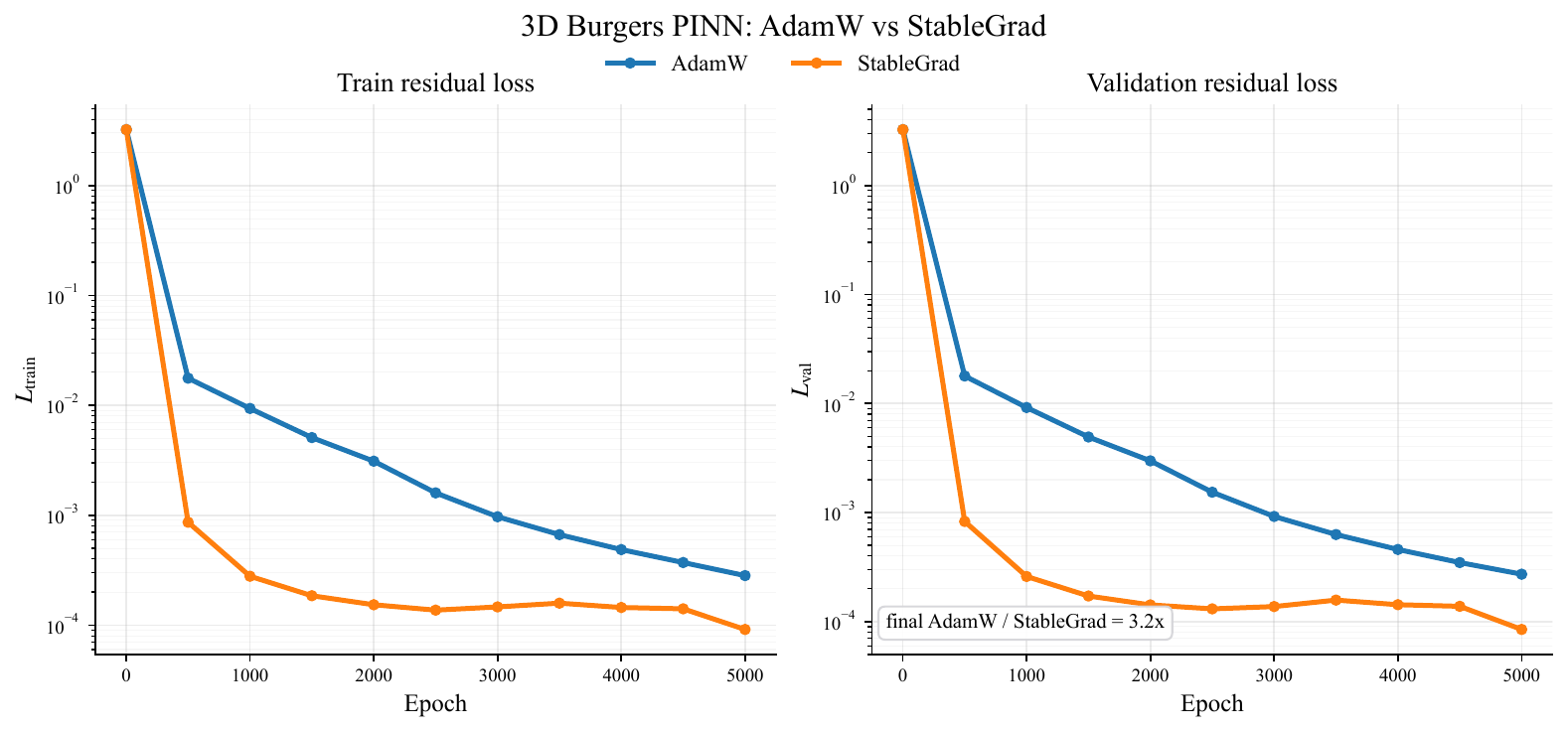}
    \caption{
    Train and validation losses for AdamW and AdamW+StableGrad on the controlled Burgers
    diagnostic run. The vertical axis is logarithmic. StableGrad reduces the loss much faster
    than AdamW and reaches, early in training, validation losses lower than the final loss
    attained by AdamW at the end of the run. The small non-monotone segment around epochs
    3000 and 3500 coincides with the checkpoints where the margin in
    Theorem~\ref{thm:local_decrease} becomes negative.
    }
    \label{fig:effective_diagnostics_loss}
\end{figure}

Theorem~\ref{thm:local_decrease} states that the StableGrad linearized step gives a larger local
decrease of \(\frac12\|r\|^2\) than the standard linearized gradient step when
\[
    \eta\lambda_{\max}(K_{\mathrm{SG}})<2
\]
and
\[
    M_{\mathrm{SG}}
    :=
    \rho_{\mathrm{SG}}
    \left(
    1-\frac{\eta}{2}\lambda_{\max}(K_{\mathrm{SG}})
    \right)-\rho
    >0.
\]
We therefore report the stability factor
\[
    s_{\mathrm{SG}}=\eta\lambda_{\max}(K_{\mathrm{SG}})
\]
and the theorem margin \(M_{\mathrm{SG}}\). Positive values of \(M_{\mathrm{SG}}\) indicate that
the sufficient condition of Theorem~\ref{thm:local_decrease} is satisfied.

To assess whether the linearized quantities are meaningful for the actual optimizer step, we also
measure the relative residual linearization error. Given the realized parameter update
\(\Delta\theta\), we compute
\[
    E_{\mathrm{lin}}
    =
    \frac{
    \|r(\theta+\Delta\theta)-r(\theta)-J\Delta\theta\|
    }{
    \|r(\theta+\Delta\theta)-r(\theta)\|+\varepsilon
    }.
\]
Small values of \(E_{\mathrm{lin}}\) indicate that the first-order residual model used in the
theory is accurate at the scale of the actual update.

\begin{table*}[t]
\centering
\caption{
Effective-dynamics diagnostics for the StableGrad run on the fixed diagnostic batch.
Here \(s_{\mathrm{SG}}=\eta\lambda_{\max}(K_{\mathrm{SG}})\), \(M_{\mathrm{SG}}\) is the margin
in Theorem~\ref{thm:local_decrease}, and \(R_{\mathrm{std}}\) is the ratio between the largest
and smallest layer-wise gradient standard deviations before and after StableGrad rescaling.
}
\label{tab:effective_diagnostics_burgers}
\begin{tabular*}{\textwidth}{@{\extracolsep{\fill}}c c c c c c c c@{}}
\toprule
Epoch
& Val. loss
& \(\rho\)
& \(\rho_{\mathrm{SG}}\)
& \(s_{\mathrm{SG}}\)
& \(M_{\mathrm{SG}}\)
& \(E_{\mathrm{lin}}\)
& \(R_{\mathrm{std}}\) raw \(\to\) scaled \\
\midrule
500  & \(8.28{\times}10^{-4}\) & 0.51  & 3.40  & 0.134 & 2.66   & \(2.63{\times}10^{-3}\) & 42.7 \(\to\) 1.00 \\
1000 & \(2.59{\times}10^{-4}\) & 2.15  & 5.89  & 0.097 & 3.45   & \(1.32{\times}10^{-3}\) & 136.0 \(\to\) 1.00 \\
2000 & \(1.42{\times}10^{-4}\) & 9.49  & 11.88 & 0.059 & 2.05   & \(1.46{\times}10^{-3}\) & 282.2 \(\to\) 1.00 \\
2500 & \(1.31{\times}10^{-4}\) & 13.03 & 13.76 & 0.050 & 0.378  & \(1.66{\times}10^{-3}\) & 316.2 \(\to\) 1.00 \\
3000 & \(1.37{\times}10^{-4}\) & 15.82 & 14.99 & 0.044 & -1.153 & \(1.77{\times}10^{-3}\) & 334.3 \(\to\) 1.00 \\
3500 & \(1.57{\times}10^{-4}\) & 13.87 & 13.43 & 0.042 & -0.722 & \(2.80{\times}10^{-3}\) & 300.1 \(\to\) 1.00 \\
4000 & \(1.43{\times}10^{-4}\) & 10.76 & 11.63 & 0.042 & 0.622  & \(4.18{\times}10^{-3}\) & 227.5 \(\to\) 1.00 \\
5000 & \(8.47{\times}10^{-5}\) & 14.19 & 14.50 & 0.034 & 0.0646 & \(2.88{\times}10^{-3}\) & 295.0 \(\to\) 1.00 \\
\bottomrule
\end{tabular*}
\end{table*}

Table~\ref{tab:effective_diagnostics_burgers} shows that the stability condition is comfortably
satisfied throughout the run once the initial transient has passed. The quantity
\(s_{\mathrm{SG}}=\eta\lambda_{\max}(K_{\mathrm{SG}})\) remains between \(3.4\times 10^{-2}\)
and \(1.34\times 10^{-1}\), far below the threshold value 2. Thus, the StableGrad effective
kernel operates well inside the stable regime covered by Theorem~\ref{thm:local_decrease}.

The theorem margin is positive for most of the reported checkpoints, including the early and final
phases of training. In those regimes, the sufficient condition of
Theorem~\ref{thm:local_decrease} holds: the StableGrad effective kernel acts more strongly on the
current residual while remaining stable. The two negative checkpoints, at epochs 3000 and 3500,
are particularly informative rather than problematic. They occur exactly when the rapid initial
descent has already saturated and the validation loss temporarily stops improving, as seen in
Figure~\ref{fig:effective_diagnostics_loss}. This is precisely what one should expect from a
local sufficient condition: when the margin becomes negative, the theorem no longer predicts an
improved local decrease, and empirically the loss indeed ceases to improve. The condition is then
recovered at later checkpoints, and the loss decreases again.

The residual linearization error remains small throughout training. After the first checkpoint,
\(E_{\mathrm{lin}}\) stays in the \(10^{-3}\) range and never exceeds \(4.18\times 10^{-3}\).
This indicates that the first-order residual model used to define \(K_{\mathrm{SG}}\),
\(\rho_{\mathrm{SG}}\), and \(\lambda_{\max}(K_{\mathrm{SG}})\) is accurate at the scale of the
actual optimizer update. The diagnostics are therefore not merely formal kernel quantities; they
provide a faithful local description of the optimization dynamics observed in the trained PINN.

Finally, the raw layer-wise gradient scales are highly imbalanced before StableGrad rescaling,
with \(R_{\mathrm{std}}\) ranging from \(42.7\) to \(334.3\) across the reported checkpoints.
StableGrad reduces this ratio to one by construction. This confirms that the method is carrying
out the intended correction: it leaves the forward residual unchanged, but equalizes the scale of
the layer-wise weight gradients before they are passed to the optimizer.

Overall, the controlled diagnostics show a clear agreement between theory and practice. StableGrad
operates in a stable regime, satisfies the sufficient condition in
Theorem~\ref{thm:local_decrease} during the phases where the loss decreases most effectively, and
temporarily violates it exactly when the empirical trajectory stops improving. The practical effect
is substantial: at epoch 5000, AdamW reaches training and validation losses of
\(2.82\times 10^{-4}\) and \(2.72\times 10^{-4}\), whereas AdamW+StableGrad reaches
\(9.14\times 10^{-5}\) and \(8.47\times 10^{-5}\). More importantly, StableGrad reaches a
validation loss below the final AdamW value already around epoch 1000, showing that the improved
effective dynamics predicted by the theory translate into a much faster reduction of the PINN
training objective in practice.

\subsection{Learning-Rate Schedule Control}
\label{app:lr_scheduler_control}

One possible interpretation of StableGrad is that it improves training mainly by inducing a larger effective step size. 
We therefore include a control experiment designed to test this explanation directly. 
Besides AdamW and AdamW+StableGrad, we train AdamW with a piecewise learning-rate multiplier chosen to mimic the spectral scale change induced by StableGrad. 
For each epoch interval, the multiplier is set from the observed ratio
\[
    m
    \approx
    \frac{
    \lambda_{\max}(K_{\mathrm{SG}})
    }{
    \lambda_{\max}(K)
    },
    \qquad
    K = JJ^\top,
    \qquad
    K_{\mathrm{SG}} = JP_{\mathrm{SG}}J^\top .
\]
Thus, if the advantage of StableGrad were only due to a larger effective learning rate, this boosted AdamW control should reproduce its behavior. 
We do not use the ratio at epoch 1, which is dominated by the initialization transient; the first interval instead uses the ratio measured at epoch 500.

\begin{table}[t]
\centering
\caption{
Piecewise learning-rate multipliers used for the AdamW scheduler control. 
Each multiplier is chosen from the observed ratio
\(\lambda_{\max}(K_{\mathrm{SG}})/\lambda_{\max}(K)\) on the corresponding diagnostic interval.
}
\label{tab:lr_scheduler_control}
\begin{tabular*}{\columnwidth}{@{\extracolsep{\fill}}l c c c c c}
\toprule
Epoch interval
& 1--500
& 501--1000
& 1001--1500
& 1501--2000
& 2001--2500 \\
LR multiplier
& 5.114545
& 4.355357
& 3.303698
& 2.603073
& 2.214707 \\
\midrule
Epoch interval
& 2501--3000
& 3001--3500
& 3501--4000
& 4001--4500
& 4501--5000 \\
LR multiplier
& 1.795869
& 1.487962
& 1.336904
& 1.139297
& 0.950362 \\
\bottomrule
\end{tabular*}
\end{table}

Figure~\ref{fig:lr_scheduler_control} compares the resulting train and validation losses. 
The scheduler control is a strong baseline: it improves clearly over standard AdamW, showing that part of the acceleration can indeed be attributed to increasing the global step scale. 
However, it does not reproduce StableGrad. 
StableGrad maintains lower residual losses throughout training: the validation loss of the scheduler control is \(6.59\times\), \(7.76\times\), \(4.38\times\), \(2.54\times\), and \(2.08\times\) higher than StableGrad at epochs 500, 1000, 2000, 3000, and 5000, respectively. 
At the final checkpoint, StableGrad reaches a validation residual loss of \(8.473\times 10^{-5}\), compared with \(1.762\times 10^{-4}\) for the scheduler control.

\begin{figure}[t]
    \centering
    \includegraphics[width=\linewidth,trim={0cm 0cm 0cm 0.8cm},clip]{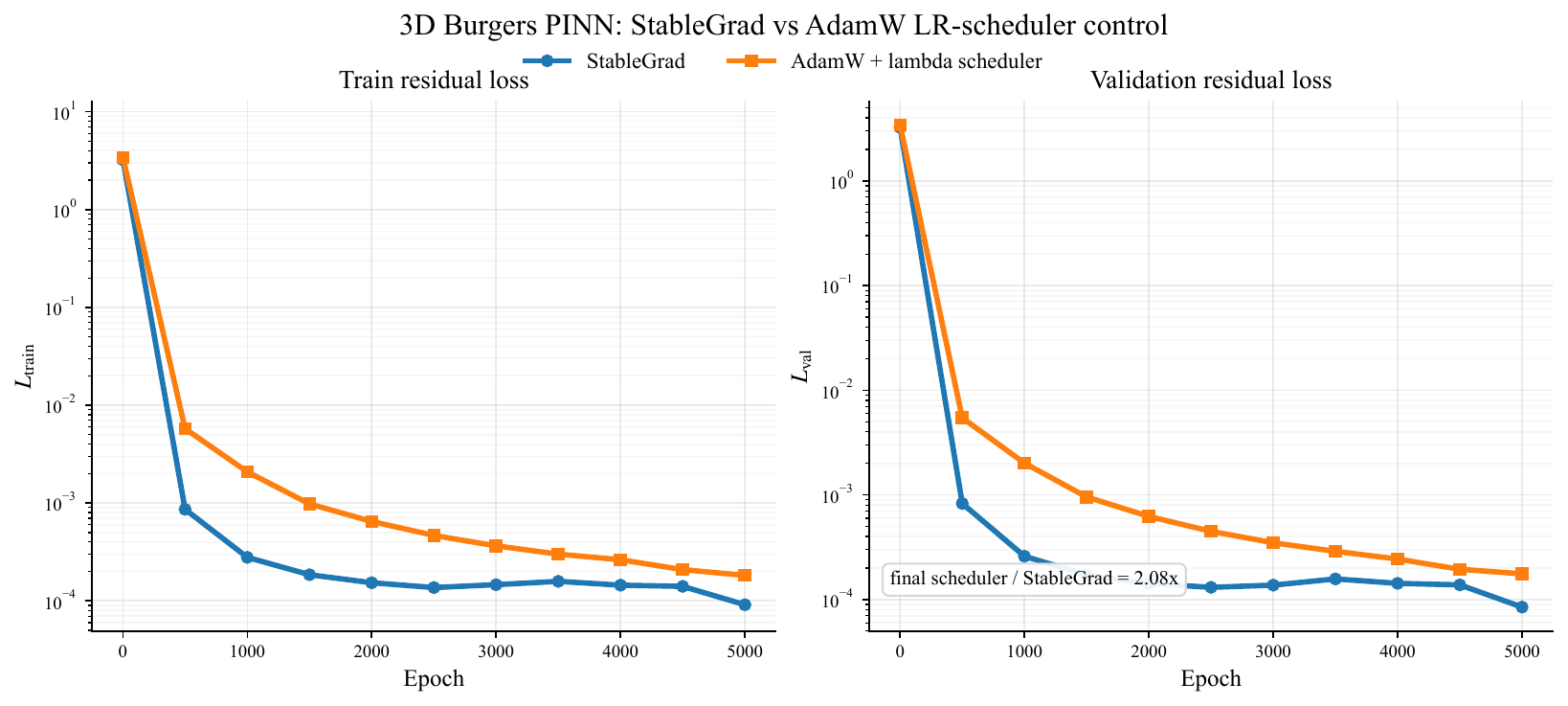}
    \caption{
    Train and validation residual losses for AdamW, AdamW with the spectral learning-rate scheduler, and AdamW+StableGrad. 
    The scheduler improves over AdamW, but does not reproduce the lower residual losses reached by StableGrad.
    }
    \label{fig:lr_scheduler_control}
\end{figure}

The difference is also visible in the update geometry. 
The scheduler changes the global step size, but it does not rebalance how the update is distributed across layers. 
At the final checkpoint, the valid relative-update ratio is \(398.7\) for AdamW with the scheduler and \(77.5\) for StableGrad, while the maximum update-energy concentration is \(0.782\) for the scheduler and \(0.350\) for StableGrad. 
Thus, the scheduler update remains much more concentrated in a small number of parameter blocks. 
StableGrad, in contrast, changes not only the amount of progress made per step, but also the layer-wise geometry of the update.

This control shows that StableGrad is not simply a learning-rate scheduler in disguise. 
Boosting the learning rate using spectral information explains part of the improvement over AdamW, but it does not recover the residual loss or the update distribution obtained by StableGrad. 
We note that the scheduler control obtains a lower field relative error in this particular run, so this experiment should not be read as a claim that StableGrad dominates every metric. 
Its purpose is narrower: it shows that the residual-loss gains and layer-wise update dynamics induced by StableGrad cannot be reduced to global learning-rate scaling alone.

\section{Additional Experimental Details}
\label{app:additional_experiments}

This section provides additional empirical evidence and reproducibility details for the experiments in the main paper. 
Appendix~\ref{app:std_wo_bn} analyzes the early activation-scale instability of BatchNorm-free EfficientNetV2-S, and Appendix~\ref{app:sign_gradient_comparison} compares StableGrad with a more aggressive sign-based gradient preprocessing baseline. 
Appendix~\ref{app:eval_details} specifies the evaluation protocols, training setups, benchmark definitions, validation procedures, and hardware requirements used for the reported CNN and PINN experiments.

\subsection{Evolution of Standard Deviation without BatchNorm}\label{app:std_wo_bn}

Figure~\ref{fig:std_evolution_without_bn} provides a magnified view of the initial training region shown in Figure~\ref{fig:cnn_training}, where the model trained without BatchNorm stops almost immediately after the beginning of training. To better understand this failure, the figure shows the evolution of the activation standard deviation across the convolutional layers of EfficientNetV2-S after removing all BatchNorm layers.

The activation standard deviations are plotted in logarithmic scale, and each curve is vertically shifted for readability. Consequently, sharp vertical movements in the plot correspond to large multiplicative changes in activation scale. This makes the sudden spikes and drops especially relevant, since they indicate abrupt changes in the numerical range of the activations.

Without BatchNorm, there is no normalization mechanism to recalibrate the scale of the intermediate activations. As training progresses, the scale of the weights can grow, which increases the magnitude of the activations produced by each layer. These activations are then passed to the following layers, so the effect can accumulate throughout the network. As a result, deeper layers tend to exhibit stronger fluctuations, since they are affected by the scale changes introduced by all preceding layers.

The figure shows that this instability appears very early in training. After fewer than 500 batches, the activation scale becomes unstable enough to produce a numerical overflow. This overflow propagates as a NaN value, after which the gradients and network weights also become NaNs. At that point, training fails and the curve corresponding to the model without BatchNorm in Figure~\ref{fig:cnn_training} stops.

\begin{figure}[ht]
    \centering
    \includegraphics[width=1\linewidth,trim={0cm 0cm 0cm 0.7cm},clip]{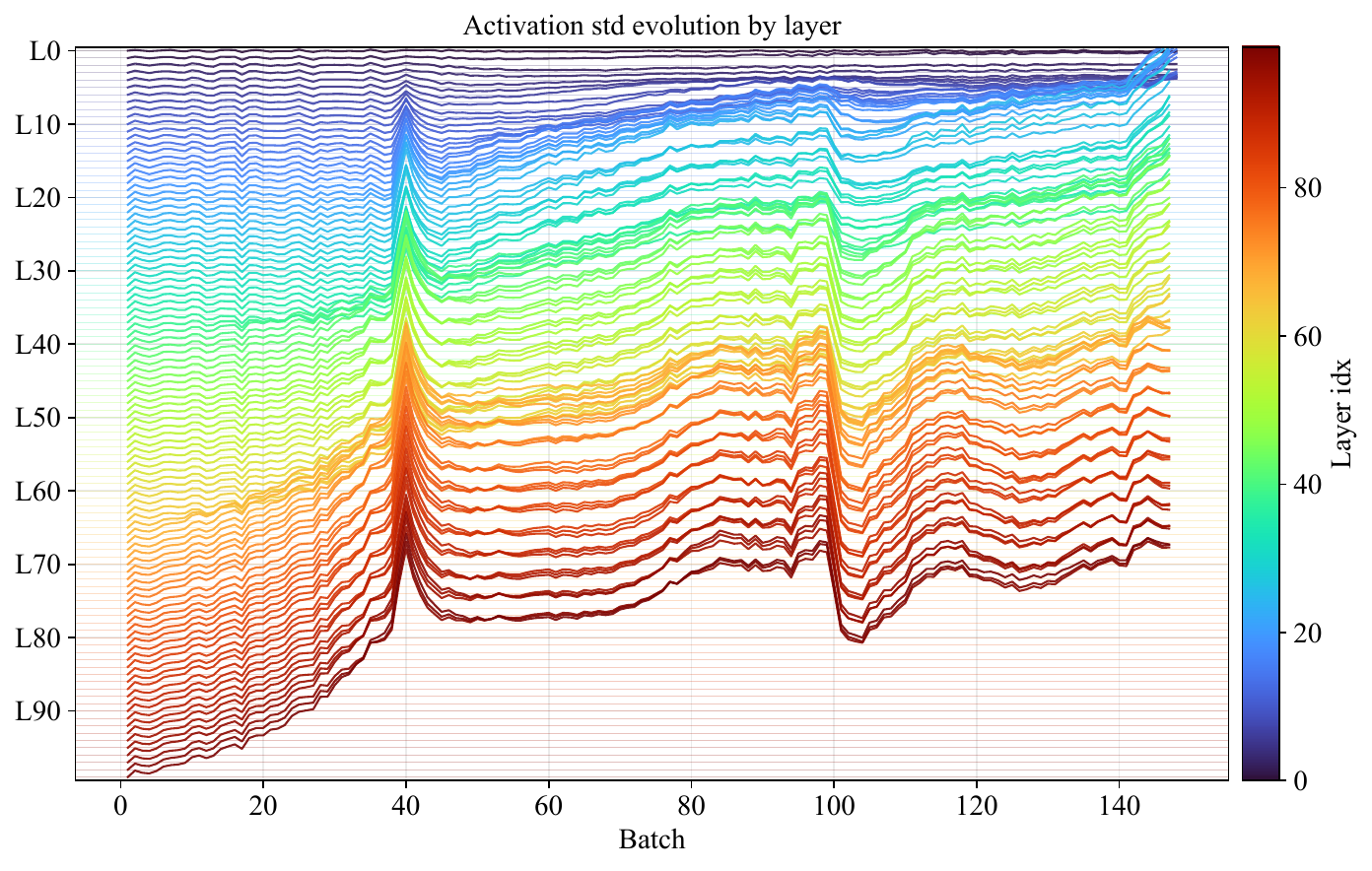}
    \caption{Magnified view of the activation-scale instability observed at the beginning of training for EfficientNetV2-S without BatchNorm. The activation standard deviation is shown in logarithmic scale across convolutional layers. Each curve corresponds to one Conv2D layer, vertically offset for readability and colored by layer index.}
    \label{fig:std_evolution_without_bn}
\end{figure}

\subsection{Comparison with Sign-Based Gradient Preprocessing}
\label{app:sign_gradient_comparison}
StableGrad modifies the gradients before the optimizer step, but it preserves their internal magnitude structure up to a layer-wise rescaling. 
A natural alternative is a more aggressive pre-optimizer transformation: replacing each gradient block \(g^\ell\) by its element-wise sign,
\[
    g^\ell \leftarrow \operatorname{sign}(g^\ell).
\]
This transformation can be viewed as an extreme form of scale homogenization, since it removes gradient-magnitude variation altogether. 
However, it also discards the relative magnitude information inside each layer, which may be important for the optimizer.

We tested this sign-based preprocessing in the same BatchNorm-free EfficientNet setting used in Section~\ref{subsec:CNNs}. 
Under the same training protocol, the sign-gradient variant failed almost immediately: training collapsed by epoch 2 and validation accuracy remained near \(1\%\), i.e. chance level on CIFAR-100. 
In contrast, StableGrad trains the same BatchNorm-free architecture stably without requiring architectural normalization or additional max-norm constraints.

This comparison shows that the benefit of StableGrad is not merely due to making all gradient scales similar. 
A naive sign transformation also removes scale variation, but it destroys too much gradient information to train the model. 
StableGrad instead equalizes layer-wise gradient scale while preserving the within-layer gradient structure passed to the optimizer.

\subsection{Evaluation Details}\label{app:eval_details}

\paragraph{CNN classification experiments.}
For the CNN classification baselines, we evaluate ResNet-50 on ImageNet-1k and EfficientNetV2-S on CIFAR-100. In both cases, images are processed at an input resolution of \(224\times224\), using random resized crops, random horizontal flips, and normalization. Training uses a batch size of \(128\). ResNet-50 is trained on ImageNet-1k for \(70\) epochs using SGD with momentum \(0.9\), learning rate \(0.1\), and weight decay \(2\times10^{-5}\). EfficientNetV2-S is trained on CIFAR-100 for \(100\) epochs using AdamW with learning rate \(10^{-3}\) and weight decay \(10^{-3}\). Both models are trained with a cosine annealing learning-rate scheduler. For each experiment, the checkpoint with the best validation top-1 accuracy is used for evaluation, and we report top-1 accuracy on the corresponding validation set.

\paragraph{Burgers equation.}
For the Burgers benchmark, we consider the one-dimensional viscous Burgers equation \(u_t + u u_x = \nu u_{xx}\), on \((x,t)\in[-1,1]\times[0,1]\), with viscosity \(\nu=10^{-4}\), initial condition \(u(x,0)=-\sin(\pi x)\), and homogeneous Dirichlet boundary conditions \(u(-1,t)=u(1,t)=0\). The PINN is a fully connected network with width \(64\), tanh activations, and the depth specified in each experiment. The model is trained with AdamW using initial learning rate \(10^{-3}\), zero weight decay, and a cosine annealing learning-rate scheduler. Baseline runs use AdamW for \(50{,}000\) optimization steps. StableGrad runs use AdamW with StableGrad for the first \(25{,}000\) steps, followed by \(25{,}000\) additional fine-tuning steps with standard AdamW. The loss is a weighted sum of the PDE residual, initial-condition, and boundary-condition losses, with weights \(1\), \(10\), and \(10\), respectively. During training, each stochastic batch contains \(100{,}000\) PDE collocation points, \(2{,}048\) initial-condition points, and \(2{,}048\) boundary-condition points. Evaluation is performed on independent validation samples and using the relative \(L^2\) error with respect to a high-resolution numerical reference solution. The reference solution is generated with a method-of-lines solver on a uniform grid of \(4096\) spatial points and \(401\) time snapshots.

\paragraph{Poisson equation.}
For the Poisson benchmark, we use fully connected PINNs with width \(64\), tanh activations, and the depth specified in each experiment. The model is trained with AdamW using initial learning rate \(10^{-3}\), zero weight decay, and a cosine annealing learning-rate scheduler. Baseline runs use AdamW for \(50{,}000\) optimization steps, while StableGrad runs use AdamW with StableGrad for the first \(25{,}000\) steps and standard AdamW for the remaining \(25{,}000\) fine-tuning steps. The loss consists of an interior PDE residual term and a softly enforced Dirichlet boundary-condition term, weighted by \(\lambda_{\mathrm{PDE}}=1\) and \(\lambda_{\mathrm{BC}}=100\), respectively. During training, each stochastic batch contains \(16{,}384\) interior residual points and \(4{,}096\) boundary points. For validation, the solution is evaluated on a \(256\times256\) grid, and the PDE and boundary losses are computed using \(65{,}536\) interior residual points and \(16{,}384\) boundary points.

\paragraph{Helmholtz equation.}
For the Helmholtz benchmark, we consider the three-dimensional problem with wave number \(k=10\pi\) and exact solution
\[
u(x,y,z)=\sin(m\pi x)\sin(m\pi y)\sin(m\pi z),
\]
with mode \(m=10\). The Helmholtz PINN uses Fourier feature inputs and SiLU activations. The input coordinates \((x,y,z)\) are augmented with sinusoidal features \(\sin(\pi f x_i)\) and \(\cos(\pi f x_i)\), for each coordinate \(x_i\in\{x,y,z\}\) and frequencies \(f=1,\ldots,12\), while also retaining the original coordinates. The resulting features are passed to a fully connected network with width \(64\) and the depth specified in each experiment. The model is trained with AdamW using learning rate \(10^{-4}\), zero weight decay, and a warm-up period of \(1{,}000\) steps. Baseline runs use AdamW for \(50{,}000\) optimization steps, while StableGrad runs use AdamW with StableGrad for the first \(25{,}000\) steps and standard AdamW for the remaining \(25{,}000\) fine-tuning steps. The PDE residual is normalized by \(k^2\), which stabilizes optimization at high wave numbers. The loss consists of the normalized PDE residual and a softly enforced Dirichlet boundary-condition term, weighted by \(\lambda_{\mathrm{PDE}}=1\) and \(\lambda_{\mathrm{BC}}=100\), respectively. During training, each stochastic batch contains \(32{,}768\) interior residual points and \(8{,}192\) boundary points. For validation, the PDE and boundary losses are computed using \(65{,}536\) interior residual points and \(16{,}384\) boundary points, and the relative \(L^2\) error is computed on a uniform grid of size \(128^3\), evaluated in chunks of \(65{,}536\) points.

Unless otherwise stated, all PINN constraints are imposed softly through penalty terms in the training objective. The numerical values reported in the tables are computed on validation data, whereas the optimization curves shown in the plots use the corresponding training losses.

\paragraph{Code availability.}
The code used to run all experiments reported in this work, together with a reusable implementation of StableGrad for other training pipelines, is available at
\href{https://github.com/anonymized/stablegrad}{\texttt{github.com/anonymized/stablegrad}}.

\paragraph{Hardware requirements.}
The experiments require a CUDA-compatible NVIDIA GPU supported by recent PyTorch releases. To provide conservative runtime estimates, each individual PINN training run can be reproduced on a single NVIDIA H100, or an equivalent accelerator, in under one hour. For the CNN experiments, each EfficientNetV2-S/CIFAR-100 training run can be reproduced in under four hours, while each full ResNet-50/ImageNet-1k training run can be reproduced in under 24 hours. These times are conservative upper bounds rather than the minimum required runtime. Full-length CNN training is only necessary to reproduce the final reported accuracies; shorter runs are sufficient to verify that the implementation trains correctly without BatchNorm.



\end{document}